\documentclass{article}

\usepackage[T1]{fontenc}
\usepackage{lmodern}

\usepackage{graphicx}
\usepackage{subfigure}
\usepackage{algorithm}
\usepackage{algorithmic}
\usepackage{amssymb}
\usepackage{amsmath}
\usepackage{amsfonts}
\usepackage{color}
\usepackage{enumitem}
\usepackage{mathtools}
\usepackage{amsthm}
\usepackage{microtype}
\usepackage{booktabs} 
\usepackage{hyperref}
\usepackage{bm}

\graphicspath{{figs/}{../figs/}{./}{../../figs/}}

\newcommand{\cmt}[1]{{\footnotesize\textcolor{red}{#1}}}
\newcommand{\todo}[1]{\cmt{(TODO: #1)}}

\newcommand{\eye}{\boldsymbol{I}}

\newcommand{\task}{\mathcal{T}}
\newcommand{\loss}{\mathcal{L}}
\newcommand{\data}{\mathcal{D}}
\newcommand{\inp}{\mathbf{x}}
\newcommand{\out}{\mathbf{y}}

\newcommand{\buffer}{\mathcal{B}}

\newcommand{\cW}{\mathcal{W}}

\newcommand{\cL}{\mathcal{L}}
\newcommand{\cO}{O}

\newcommand{\bR}{\mathbb{R}}
\newcommand{\bE}{\mathbb{E}}
\newcommand{\bP}{\mathbb{P}}

\newtheorem{theorem}{Theorem}
\newtheorem{lemma}{Lemma}
\newtheorem{definition}{Definition}
\newtheorem{assumption}{Assumption}
\newtheorem{corollary}{Corollary}

\newcommand{\param}{{\mathbf{w}}}               
\newcommand{\udparam}{\tilde{\param}}     
\newcommand{\fn}{f}                  
\newcommand{\udfn}{\tilde{\fn}}        
\newcommand{\fnht}{\hat{\fn}}        
\newcommand{\ud}{{\bm{U}}}                  
\newcommand{\dloss}{\ell}            

\newcommand{\px}{\bm{\theta}}    
\newcommand{\py}{\bm{\phi}}    
\newcommand{\pz}{\bm{\psi}}    
\newcommand{\step}{k} 
\newcommand{\udpx}{\tilde{\px}}
\newcommand{\udpy}{\tilde{\py}}

\newcommand{\Aq}{{\bm{A}}}
\newcommand{\bq}{{\bm{b}}}

\newcommand{\zero}{\bm{0}}

\newcommand{\bfn}{\bm{\varphi}}
\newcommand{\jacobian}{\nabla}

\usepackage[accepted]{icml2019_edited}

\icmltitlerunning{Online Meta-Learning}

\begin{document}

\twocolumn[
\icmltitle{Online Meta-Learning}

\icmlsetsymbol{equal}{*}

\begin{icmlauthorlist}
\icmlauthor{Chelsea Finn}{equal,berkeley}
\icmlauthor{Aravind Rajeswaran}{equal,uw}
\icmlauthor{Sham Kakade}{uw}
\icmlauthor{Sergey Levine}{berkeley}
\end{icmlauthorlist}

\icmlaffiliation{uw}{University of Washington}
\icmlaffiliation{berkeley}{University of California, Berkeley}

\icmlcorrespondingauthor{Chelsea Finn}{{\tt cbfinn@cs.stanford.edu}}
\icmlcorrespondingauthor{Aravind Rajeswaran}{{\tt aravraj@cs.washington.edu}}

\vskip 0.3in
]

\printAffiliationsAndNotice{\icmlEqualContribution}

\begin{abstract}
A central capability of intelligent systems is the ability to continuously build upon previous experiences to speed up and enhance learning of new tasks. Two distinct research paradigms have studied this question. Meta-learning views this problem as learning a prior over model parameters that is amenable for fast adaptation on a new task, but typically assumes the set of tasks are available together as a batch. In contrast, online (regret based) learning considers a sequential setting in which problems are revealed one after the other, but conventionally train only a single model without any task-specific adaptation.
This work introduces an online meta-learning setting, which merges ideas from both the aforementioned paradigms to better capture the spirit and practice of continual lifelong learning. We propose the follow the meta leader (FTML) algorithm which extends the MAML algorithm to this setting. Theoretically, this work provides an  $\cO(\log T)$ regret guarantee with  one additional higher order smoothness assumption (in comparison to the standard online setting). 
Our experimental evaluation on three different large-scale tasks suggest that the proposed algorithm significantly outperforms alternatives based on traditional online learning approaches.
\end{abstract}
\vspace*{-10pt}

\section{Introduction}
\label{sec:intro}

Two distinct research paradigms have studied how prior tasks or experiences can be used by an agent to inform future learning.
Meta-learning~\cite{schmidhuber1987,matchingnets,maml} casts this as the problem of \emph{learning to learn}, where past experience is used to acquire a prior over model parameters or a learning procedure, and typically studies a setting where a set of meta-training tasks are made available together upfront. In contrast, online learning~\cite{Hannan, CesaBianchi2006PLA} considers a sequential setting where tasks are revealed one after another, but aims to attain zero-shot generalization without any task-specific adaptation. We argue that neither setting is ideal for studying continual lifelong learning. Meta-learning deals with learning to learn, but neglects the sequential and non-stationary aspects of the problem. Online learning offers an appealing theoretical framework, but does not generally consider how past experience can accelerate adaptation to a new task. In this work, we motivate and present the \emph{online meta-learning} problem setting, where the agent simultaneously uses past experiences in a sequential setting to learn good priors, and also adapt quickly to the current task at hand.

\begin{figure}[b!]
    \centering
    \includegraphics[height=2.5cm]{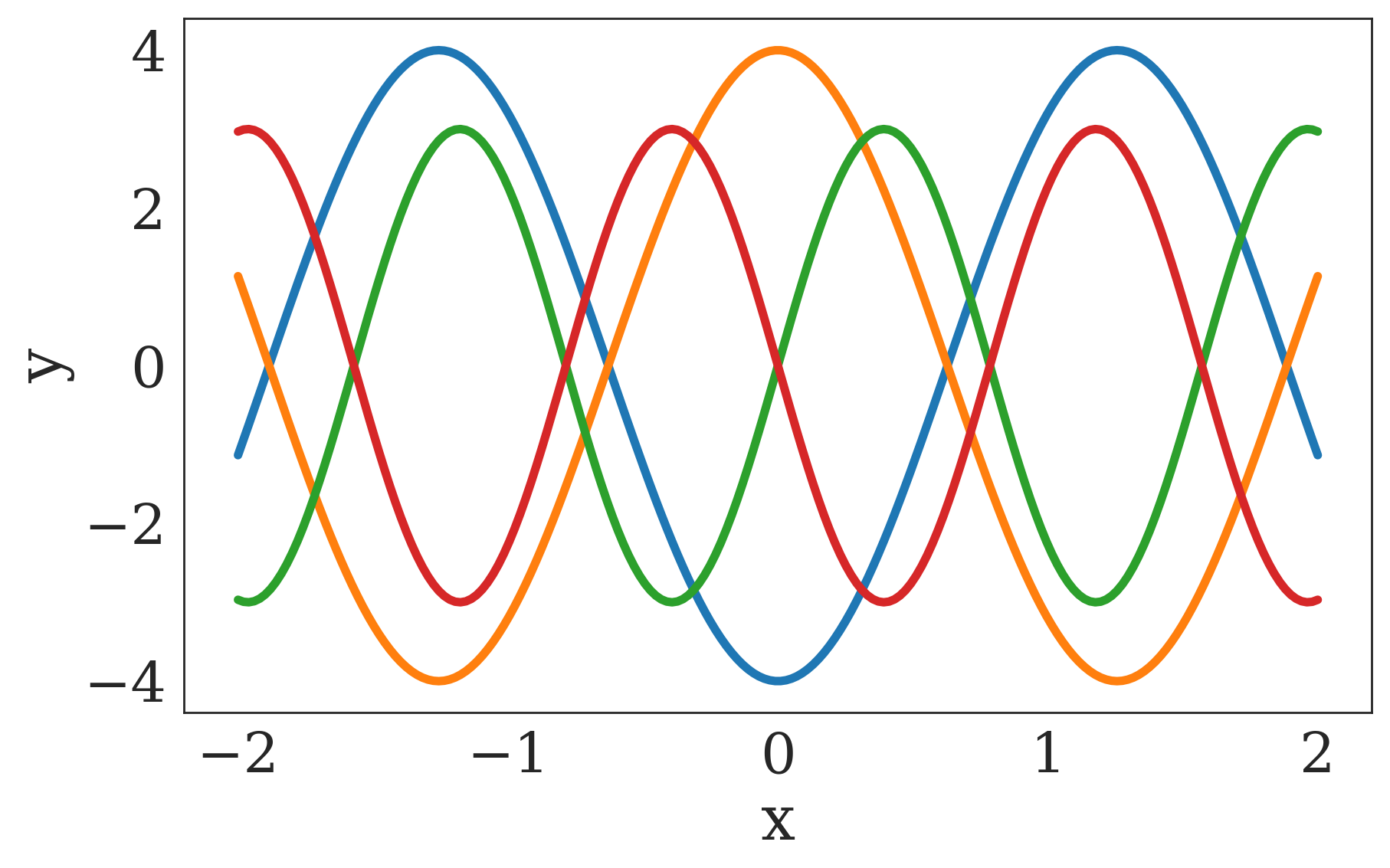}
    \includegraphics[height=2.2cm]{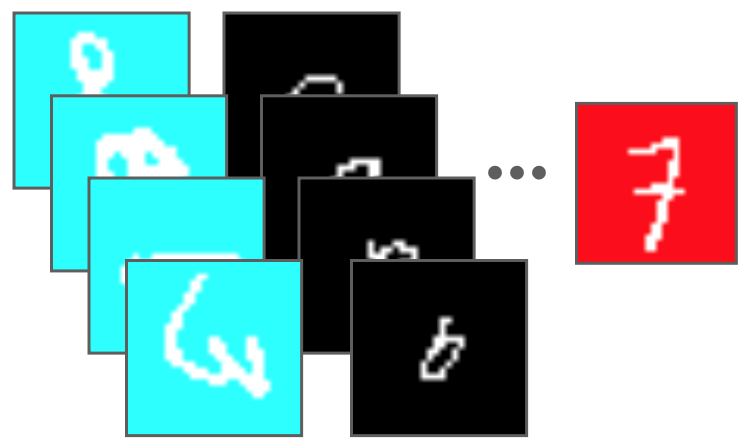}
    \vspace{-0.3cm}
    \caption{(left) sinusoid functions and (right) colored MNIST}
    \label{fig:teaser}
\end{figure}

As an example, Figure~\ref{fig:teaser} shows a family of sinusoids. Imagine that each task is a regression problem from $x$ to $y$ corresponding to {\em one} sinusoid. When presented with data from a large collection of such tasks, a na\"{i}ve approach that does not consider the task structure would collectively use all the data, and learn a prior that corresponds to the model $y=0$. An algorithm that understands the underlying structure would recognize that each curve in the family is a (different) sinusoid, and would therefore attempt to identify, for a new batch of data, which sinusoid it corresponds to. As another example where na\"{i}ve training on prior tasks fails, Figure~\ref{fig:teaser} also shows colored MNIST digits with different backgrounds. Suppose we've seen MNIST digits with various colored backgrounds, and then observe a ``7'' on a new color. We might conclude from training on all of the data seen so far that all digits with that color must all be ``7.''
In fact, this is an optimal conclusion from a purely statistical standpoint. 
However, if we understand that the data is divided into different tasks, and take note of the fact that each task has a different color, a better conclusion is that the color is irrelevant. Training on all of the data together, or only on the new data, does not achieve this goal.

Meta-learning offers an appealing solution: by learning how to learn from past tasks, we can make use of task structure and extract information from the data that both allows us to succeed on the current task and adapt to new tasks more quickly. However, typical meta learning approaches assume that a sufficiently large set of tasks are made available upfront for meta-training. In the real world, tasks are likely available only sequentially, as the agent is learning in the world, and also from a non-stationary distribution. 
By recasting meta-learning in a sequential or online setting, that does not make strong distributional assumptions, we can enable faster learning on new tasks as they are presented.

\textbf{Our contributions:}  \
In this work, we formulate the online meta-learning problem setting and present the {\em follow the meta-leader (FTML)} algorithm. This extends the MAML algorithm to the online meta-learning setting, and is analogous to follow the leader in online learning.
We analyze FTML and show that it enjoys a $\cO(\log T)$ regret guarantee when competing with the best meta-learner in hindsight. In this endeavor, we also provide the first set of results (under any assumptions) where MAML-like objective functions can be provably and efficiently optimized.
We also develop a practical form of FTML that can be used effectively with deep neural networks on large scale tasks, and show that it significantly outperforms prior methods. The experiments involve vision-based sequential learning tasks with the MNIST, CIFAR-100, and PASCAL 3D+ datasets.
\section{Foundations}
\label{sec:foundations}

Before introducing online meta-learning, we first briefly summarize the foundations of meta-learning, the model-agnostic meta-learning (MAML) algorithm, and online learning. To illustrate the differences in setting and algorithms, we will use the running example of few-shot learning, which we describe below first. We emphasize that online learning, MAML, and the online meta-learning formulations have a broader scope than few-shot supervised learning. We use the few-shot supervised learning example primarily for illustration.

\subsection{Few-Shot Learning} 
\label{sec:few_shot_learning}
In the few-shot supervised learning setting~\citep{mann}, we are interested in a family of tasks, where each task $\task$ is associated with a notional and infinite-size population of input-output pairs. In the few-shot learning, the goal is to learn a task while accessing only a small, finite-size labeled dataset \hbox{$\data_i := \{\inp_i, \out_i \}$} corresponding to task $\task_i$. If we have a predictive model, $\bm{h}(\cdot; \param)$, with parameters $\param$, the population risk of the model is 
$$
\fn_i(\param) \coloneqq \bE_{(\inp, \out) \sim \task_i} [ \dloss(\inp, \out, \param) ],
$$
where the expectation is defined over the task population and $\dloss$ is a loss function, such as the square loss or cross-entropy between the model prediction and the correct label. An example of $\dloss$ corresponding to squared error loss is
$$
\dloss(\inp, \out, \param) = || \out - \bm{h}(\inp ; \param)  ||^2.
$$
Let $\loss(\data_i, \param)$ represent the average loss on the dataset $\data_i$.
Being able to effectively minimize $\fn_i(\param)$ is likely hard if we rely only on $\data_i$ due to the small size of the dataset. However, we are exposed to many such tasks from the family --- either in sequence or as a batch, depending on the setting. By being able to draw upon the multiplicity of tasks, we may hope to perform better, as for example demonstrated in the meta-learning literature.

\subsection{Meta-Learning and MAML}
\label{sec:maml_foundations}
Meta-learning, or learning to learn~\citep{schmidhuber1987}, aims to effectively bootstrap from a set of tasks to learn faster on a new task. It is assumed that tasks are drawn from a fixed distribution, $\task \sim \bP(\task)$. At meta-training time, $M$ tasks $\{ \task_i \}_{i=1}^M$ are drawn from this distribution and datasets corresponding to them are made available to the agent. 
At deployment time, we are faced with a new test task $\task_{j} \sim \bP(\task)$, for which we are again presented with a small labeled dataset $\data_{j} := \{\inp_{j}, \out_{j} \}$. Meta-learning algorithms attempt to find a model using the $M$ training tasks, such that when $\data_{j}$ is revealed from the test task, the model can be quickly updated to minimize $\fn_j(\param)$.

Model-agnostic meta-learning (MAML)~\cite{maml} does this by learning an initial set of parameters $\param_{\mathrm{MAML}}$, such that at meta-test time, performing a few steps of gradient descent from $\param_{\mathrm{MAML}}$ using $\data_{j}$ minimizes $\fn_j(\cdot)$. To get such an initialization, at meta-training time, MAML solves the optimization problem:
\begin{equation} \label{eq:maml_optimization}
    \param_{\mathrm{MAML}} := \arg \min_\param \ \frac{1}{M} \sum_{i=1}^M \fn_i \big( \param - \alpha \nabla \fnht_i(\param) \big).
\end{equation}
The inner gradient $\nabla \fnht_i(\param)$ is based on a small mini-batch of data from $\data_i$. Hence, MAML optimizes for few-shot generalization.
Note that the optimization problem is subtle: we have a gradient descent update step embedded in the actual objective function. Regardless,
\citet{maml} show that gradient-based methods can be used on this optimization objective with existing automatic differentiation libraries.
Stochastic optimization techniques are used to solve the optimization problem in Eq.~\ref{eq:maml_optimization} since the population risk is not known directly. At meta-test time, the solution to Eq.~\ref{eq:maml_optimization} is fine-tuned as: $\param_{j} \leftarrow \param_{\mathrm{MAML}} - \alpha \nabla \fnht_j(\param_{\mathrm{MAML}})$
with the gradient obtained using $\data_{j}$. 
Meta-training can be interpreted as learning a prior over model parameters, and fine-tuning as inference~\cite{erin}.

MAML and other meta-learning algorithms (see Section~\ref{sec:related_works}) are not directly applicable to sequential settings for two reasons. First, they have two distinct phases: meta-training and meta-testing or deployment. We would like the algorithms to work in a continuous learning fashion. Second, meta-learning methods generally assume that the tasks come from some fixed distribution, whereas we would like methods that work for non-stationary task distributions.

\subsection{Online Learning}

In the online learning setting, an agent faces a sequence of loss functions $\lbrace \fn_t \rbrace_{t=1}^\infty$, one in each round $t$. These functions need not be drawn from a fixed distribution, and could even be chosen adversarially over time. The goal for the learner is to sequentially decide on model parameters $\lbrace \param_t \rbrace_{t=1}^\infty$ that perform well on the loss sequence. In particular, the standard objective is to minimize some notion of regret defined as the difference between our learner's loss, $ \sum_{t=1}^T \fn_t(\param_t)$, and the best performance achievable by some family of methods (comparator class). The most standard notion of regret is to compare to the cumulative loss of the best {\em fixed} model in hindsight:
\begin{equation}
    \text{Regret}_T = \sum_{t=1}^T \fn_t (\param_t) - \min_\param \sum_{t=1}^T \fn_t(\param).
\end{equation}
The goal in online learning is to design algorithms such that this regret grows with $T$ as slowly as possible. In particular, an agent (algorithm) whose regret grows sub-linearly in $T$ is non-trivially learning and adapting.
One of the simplest algorithms in this setting is follow the leader (FTL)~\cite{Hannan}, which updates the parameters as: 
$$
\param_{t+1} = \arg \min_{\param} \ \sum_{\step=1}^t \fn_\step(\param). 
$$
FTL enjoys strong performance guarantees depending on the properties of the loss function, and some variants use additional regularization to improve stability~\cite{ShaiBook}.
For the few-shot supervised learning example, FTL would consolidate all the data from the prior stream of tasks into a single large dataset and fit a single model to this dataset. As alluded to in Section~\ref{sec:intro}, and as we find in our empirical evaluation in Section~\ref{sec:experiments},
such a ``joint training'' approach may not learn effective models. To overcome this issue, we may desire a more adaptive notion of a comparator class, and algorithms that have low regret against such a comparator, as we will discuss next.

\section{The Online Meta-Learning Problem}
\label{sec:problem_setting}

We consider a general sequential setting where an agent is faced with tasks one after another. Each of these tasks correspond to a {\em round}, denoted by $t$. In each round, the goal of the learner is to determine model parameters~$\param_t$ that perform well for the corresponding task at that round. This is monitored by \hbox{$\fn_t: \param \in \cW \rightarrow \bR$}, which we would like to be minimized.
Crucially, we consider a setting where the agent can perform some local {\em task-specific} updates to the model before it is deployed and evaluated in each round. This is realized through an update procedure, which at every round $t$ is a mapping $\ud_t: \param \in \cW \rightarrow \udparam \in \cW$. This procedure takes as input $\param$ and returns $\udparam$ that performs better on $\fn_t$. One example for $\ud_t$ is a step of gradient descent \cite{maml}: 
$$
\ud_t(\param) = \param - \alpha \nabla \fnht_t(\param).
$$
Here, as specified in Section~\ref{sec:maml_foundations}, $\nabla \fnht_t$ is potentially an approximate gradient of $\fn_t$, as for example obtained using a mini-batch of data from the task at round $t$. The overall protocol for the setting is as follows:
\vspace*{-10pt}
\begin{enumerate}
\itemsep0em
    \item At round $t$, the agent chooses a model defined by $\param_t$.
    \item The world simultaneously chooses task defined by $f_t$.
    \item The agent obtains access to the update procedure $\ud_t$, and uses it to update parameters as $\udparam_t = \ud_t(\param_t)$.
    \item The agent incurs loss $\fn_t(\udparam_t)$. Advance to round $t+1$.
\end{enumerate}
The goal for the agent is to minimize regret over the rounds. A highly ambitious comparator is the best meta-learned model in hindsight. Let $\{ \param_t \}_{t=1}^T$ be the sequence of models generated by the algorithm. Then, the regret we consider is: 
\begin{equation} \label{eq:oml_regret}
    \text{Regret}_T = \sum_{t=1}^T \fn_t \big( \ud_t (\param_t) \big) - \min_\param \sum_{t=1}^T \fn_t \big( \ud_t (\param) \big).
\end{equation}
Notice that we allow the comparator to adapt locally to each task at hand; thus the comparator has strictly more capabilities than the learning agent, since it is presented with all the task functions in batch mode. Here, again, achieving sublinear regret suggests that the agent is improving over time and is competitive with the best meta-learner in hindsight. As discussed earlier, in the batch setting, when faced with multiple tasks, meta-learning performs significantly better than training a single model for all the tasks. Thus, we may hope that learning sequentially, but still being competitive with the best meta-learner in hindsight, provides a significant leap in continual learning.

\section{Algorithm and Analysis}
\label{sec:algos_and_analysis}

In this section, we outline the {\em follow the meta leader} (FTML) algorithm and provide an analysis of its behavior.

\subsection{Follow the Meta Leader}
One of the most successful algorithms in online learning is follow the leader~\cite{Hannan, Kalai2005EfficientAF}, which enjoys strong performance guarantees for smooth and convex functions. Taking inspiration from its form, we propose the FTML algorithm template which updates model parameters as:
\begin{equation} \label{eq:ftml}
    \param_{t+1} = \arg \min_{\param} \ \sum_{\step=1}^t \fn_\step \big( \ud_\step (\param) \big).
\end{equation}
This can be interpreted as the agent playing the best meta-learner in hindsight if the learning process were to stop at round $t$. In the remainder of this section, we will show that under standard assumptions on the losses, and just one additional assumption on higher order smoothness, this algorithm has strong regret guarantees. In practice, we may not have full access to $\fn_\step(\cdot)$, such as when it is the population risk and we only have a finite size dataset. In such cases, we will draw upon stochastic approximation algorithms to solve the optimization problem in Eq.~(\ref{eq:ftml}).

\subsection{Assumptions}
We make the following assumptions about each loss function in the learning problem for all $t$. Let $\px$ and $\py$ represent two arbitrary choices of {\em model parameters}.

\begin{assumption}
($C^2$-smoothness) \\
\vspace*{-20pt}
\begin{enumerate}[leftmargin=*]
    \itemsep0em
    \item (Lipschitz in function value) $\fn$ has gradients  bounded by $G$, i.e. $|| \nabla \fn(\px) || \leq G \ \forall \ \px$. This is equivalent to $\fn$ being $G-$Lipschitz.
    \item (Lipschitz gradient) $\fn$ is $\beta-$smooth, i.e. \\ \hbox{$|| \nabla \fn(\px) - \nabla \fn(\py) || \leq \beta ||\px-\py|| \ \forall (\px,\py) $.}
    \item (Lipschitz Hessian) $\fn$ has $\rho-$Lipschitz Hessians, i.e. \\ \hbox{$ || \nabla^2 \fn(\px) - \nabla^2 \fn(\py) || \leq \rho ||\px-\py|| \ \forall (\px,\py) $.} 
\end{enumerate}
\end{assumption}

\begin{assumption}
(Strong convexity) 
Suppose that $\fn$ is convex. Furthermore, suppose $\fn$ is $\mu-$strongly convex, i.e. \hbox{$ || \nabla \fn(\px) - \nabla \fn(\py) || \geq \mu||\px-\py|| $.}
\end{assumption}

These assumptions are largely standard in online learning, in various settings~\cite{CesaBianchi2006PLA}, except~1.3. Examples where these assumptions hold include logistic regression and $L2$ regression over a bounded domain. Assumption 1.3 is a statement about the higher order smoothness of functions which is common in non-convex analysis~\cite{Nesterov2006CubicRO, Jin2017HowTE}. In our setting, it allows us to characterize the landscape of the MAML-like function which has a gradient update step embedded within it. Importantly, these assumptions {\em do not} trivialize the meta-learning setting. A clear difference in performance between meta-learning and joint training can be observed even in the case where $\fn(\cdot)$ are quadratic functions, which correspond to the simplest strongly convex setting. See Appendix~\ref{app:quadratic_example} for an example illustration.

\vspace{-0.1cm}
\subsection{Analysis}
\label{sec:grad_step_analysis}
\vspace{-0.1cm}

We analyze the FTML algorithm when the update procedure is a single step of gradient descent, as in the formulation of MAML. Concretely, the update procedure we consider is \hbox{$\ud_t(\param) = \param - \alpha \nabla \fnht_t(\param).$} For this update rule, we first state our main theorem below.

\begin{theorem} \label{thm:convexity}
Suppose $\fn$ and $\fnht: \bR^d \rightarrow \bR$ satisfy assumptions 1 and 2. Let $\udfn$ be the function evaluated after a one step gradient update procedure, i.e.
\vspace{-0.15cm}
$$
\vspace{-0.15cm}
\udfn(\param) \coloneqq \fn \big( \param - \alpha \nabla \fnht(\param) \big).
$$
If the step size is selected as $\alpha \leq \min \{ \frac{1}{2\beta}, \frac{\mu}{8 \rho G} \}$, then $\udfn$ is convex. Furthermore, it is also $\tilde{\beta}=9\beta/8$ smooth and $\tilde{\mu} = \mu/8$ strongly convex.
\vspace{-0.1cm}
\end{theorem}

\begin{proof}
See Appendix~\ref{app:proofs}.
\end{proof}

The following corollary is now immediate.
\begin{corollary}
(inherited convexity for the MAML objective) If $\{ \fn_i, \fnht_i \}_{i=1}^K$ satisfy assumptions~1 and 2, then the MAML optimization problem,
\vspace{-0.25cm}
$$
\vspace{-0.15cm}
\underset{\param}{\text{minimize}} \ \frac{1}{M} \sum_{i=1}^M \fn_i \big( \param - \alpha \nabla \fnht_i(\param) \big),
$$
with $\alpha \leq \min \lbrace \frac{1}{2\beta}, \frac{\mu}{8 \rho G} \rbrace$
is convex. Furthermore, it is $9\beta/8$-smooth and  $\mu/8$-strongly convex.
\vspace{-0.1cm}
\end{corollary}

Since the objective function is convex, we may expect first-order optimization methods to be effective, since gradients can be efficiently computed with standard automatic differentiation libraries (as discussed in~\citet{maml}). In fact, this work provides the first set of results (under any assumptions) under which MAML-like objective function can be provably and efficiently optimized.

Another immediate corollary of our main theorem is that FTML now enjoys the same regret guarantees (up to constant factors) as FTL does in the comparable setting (with strongly convex losses).

\begin{corollary}
\label{cor:bound}
(inherited regret bound for FTML) Suppose that for all $t$,  $\fn_t$ and $\fnht_t$ satisfy assumptions 1 and 2. Suppose that the update procedure in FTML (Eq.~\ref{eq:ftml}) is chosen as \hbox{$\ud_t(\param) = \param - \alpha \nabla \fnht_t(\param)$} with $\alpha \leq \min \lbrace \frac{1}{2\beta}, \frac{\mu}{8 \rho G} \rbrace$. Then, FTML enjoys the following regret guarantee
\vspace{-0.15cm}
$$
\vspace{-0.15cm}
\sum_{t=1}^T \fn_t \big( \ud_t(\param_t) \big) - \min_{\param} \sum_{t=1}^T \fn_t \big( \ud_t(\param) \big) = \cO \bigg( \frac{32G^2}{\mu} \log T \bigg)
$$
\end{corollary}
\begin{proof}
From Theorem~\ref{thm:convexity}, we have that each function $\udfn_t(\param) = \fn_t (\ud_t(\param))$ is $\tilde{\mu} = \mu/8$ strongly convex. The FTML algorithm is identical to FTL on the sequence of loss functions $\lbrace \udfn_t \rbrace_{t=1}^T$, which has a $\cO(\frac{4G^2}{\tilde{\mu}} \log T)$ regret guarantee~(see \citet{CesaBianchi2006PLA} Theorem 3.1). Using $\tilde{\mu} = \mu/8$ completes the proof.
\end{proof}

More generally, our main theorem implies that there exists a large family of online meta-learning algorithms that enjoy sub-linear regret, based on the inherited smoothness and strong convexity of $\udfn(\cdot)$. See~\citet{HazanOCOBook,ShaiBook,ShalevShwartz2008MindTD} for algorithmic templates to derive sub-linear regret based algorithms.

\section{Practical Online Meta-Learning Algorithm}
\label{sec:real_algorithm}

\begin{figure*}[ttt!]
\begin{minipage}[t]{0.5\textwidth}
\begin{algorithm}[H]
\caption{Online Meta-Learning with FTML}
\label{alg:mamlonline}
\begin{algorithmic}[1]
{\footnotesize
\STATE {\bfseries Input:} Performance threshold of proficiency, $\gamma$
\STATE randomly initialize $\param_1$
\STATE initialize the task buffer as empty, $\buffer \leftarrow [~]$
\FOR{$t = 1,\dots$}
\STATE initialize $\data_{t} = \emptyset$
\STATE Add $\buffer \leftarrow \buffer + [~\task_t~]$ 
\WHILE{$| \data_{\task_t}| < N$}
\STATE Append batch of $n$ new datapoints $ \left\{ \left( \inp, \out \right) \right\}$ to $\data_{t}$
\STATE $\param_{t} \leftarrow $ \texttt{Meta-Update}$(\param_t, \buffer, t)$
\STATE $\udparam_t \leftarrow  \texttt{Update-Procedure}\left(\param_t, \data_{t} \right)$
\IF{ $\loss \left( \data^\text{test}_{t}, \udparam_t \right)< \gamma$}
\STATE Record efficiency for task $\task_t$ as $| \data_{t}|$ datapoints
\ENDIF
\ENDWHILE
\STATE Record final performance of $\udparam_t$ on test set $\data_t^\text{test}$ for task $t$.
\STATE $\param_{t+1} \leftarrow \param_t$
\ENDFOR
}
\end{algorithmic}
\end{algorithm}
\end{minipage}
\begin{minipage}[t]{0.5\textwidth}
\begin{algorithm}[H]
\caption{FTML Subroutines}
\label{alg:maml}
\begin{algorithmic}[1]
{\footnotesize
\STATE {\bfseries Input:} 
Hyperparameters parameters $\alpha, \eta$
\FUNCTION{\texttt{Meta-Update}$(\param, \buffer, t)$ \ \  }
\FOR{$n_\text{m} = 1,\dots, N_\text{meta} $ steps}
\STATE Sample task $\task_k$: $k \sim \nu^t(\cdot)$ {~~\em // (or a minibatch of tasks) }
\STATE Sample minibatches $\data_k^\text{tr}$, $\data_k^\text{val}$ uniformly from $\data_k$
\STATE Compute gradient $\mathbf{g}_t$ using $\data_k^\text{tr}$, $\data_k^\text{val}$, and Eq.~\ref{eq:practical_ftml_update}
\STATE Update parameters $\param \leftarrow \param - \eta \ \mathbf{g}_t$ \ \ {~\em // (or use Adam)}
\ENDFOR
    \STATE Return $\param$
\ENDFUNCTION
\FUNCTION{\texttt{Update-Procedure}$(\param$, $\data)$}
\STATE Initialize $\udparam \leftarrow \param$
\FOR{$n_\text{g} = 1,\dots, N_\text{grad} $ steps}
\STATE $\udparam \leftarrow \udparam - \alpha \nabla \cL(\data, \udparam)$
\ENDFOR
\STATE Return $\udparam$
\ENDFUNCTION
}
\end{algorithmic}
\end{algorithm}
\end{minipage}
\vspace{-0.3cm}
\end{figure*}

In the previous section, we derived a theoretically principled algorithm for convex losses. However, many problems of interest in machine learning and deep learning have a non-convex loss landscape, where theoretical analysis is challenging. Nevertheless, algorithms originally developed for convex losses like gradient descent and AdaGrad~\cite{duchi2011adaptive} have shown promising results in practical non-convex settings. Taking inspiration from these successes, we describe a practical instantiation of FTML in this section, and empirically evaluate its performance in Section~\ref{sec:experiments}.

The main considerations when adapting the FTML algorithm to few-shot supervised learning with high capacity neural network models are: (a) the optimization problem in Eq.~(\ref{eq:ftml}) has no closed form solution, and (b) we do not have access to the population risk $\fn_t$ but only a subset of the data. To overcome both these limitations, we can use iterative stochastic optimization algorithms. Specifically, by adapting the MAML algorithm~\cite{maml}, we can use stochastic gradient descent with a minibatch $\data_k^\text{tr}$ as the update rule, and stochastic gradient descent with an independently-sampled minibatch $\data_k^\text{val}$ to optimize the parameters. The gradient computation is described below:

\begin{equation} \label{eq:practical_ftml_update}
\begin{aligned}
    \mathbf{g}_t(\param) & = \nabla_\param \
    \mathbb{E}_{\step \sim \nu^t}
    \loss \big( \data_\step^{\text{val}}, \ud_\step(\param) \big), \ \text{ where} \\
    \ud_\step(\param) & \equiv \param - \alpha \ 
    \nabla_\param \ \loss \big( \data_\step^{\text{tr}}, \param \big)
\end{aligned}
\end{equation}

Here, $\nu^t(\cdot)$ denotes a sampling distribution for the previously seen tasks $\task_1, ..., \task_t$. In our experiments, we use the uniform distribution, $\nu^t \equiv P(k)=1/t~ \forall k=\{1,2,\ldots t\}$, but other sampling distributions can be used if required.
Recall that $\loss(\data, \param)$ is the loss function (e.g. cross-entropy) averaged over the datapoints $(\inp, \out) \in \data$ for the model with parameters $\param$. 
Using independently sampled minibatches $\data^\text{tr}$ and $\data^\text{val}$ minimizes interaction between the inner gradient update $\ud_t$ and the outer optimization (Eq.~\ref{eq:ftml}), which is performed using the gradient above~($\mathbf{g}_t$) in conjunction with Adam~\citep{adam}. While $\ud_t$ in Eq.~\ref{eq:practical_ftml_update} includes only one gradient step, we observed that it is beneficial to take multiple gradient steps in the inner loop (i.e., in $\ud_t$), which is consistent with prior works~\cite{maml, erin, antoniou2018train, Shaban2018TruncatedBF}.

Now that we have derived the gradient, the overall algorithm proceeds as follows. We first initialize a task buffer $\buffer=[~]$. When presented with a new task at round $t$, we add task $\task_t$ to $\buffer$ and initialize a task-specific dataset $\data_t=[~]$, which is appended to as data incrementally arrives for task $\task_t$. As new data arrives for task $\task_t$, we iteratively compute and apply the gradient in Eq.~\ref{eq:practical_ftml_update}, which uses data from all tasks seen so far. Once all of the data (finite-size) has arrived for $\task_t$, we move on to task $\task_{t+1}$. This procedure is further described in Algorithm~\ref{alg:mamlonline}, including the evaluation, which we discuss next.

To evaluate performance of the model at any point within a particular round $t$, we first update the model as using all of the data ($\data_{t}$) seen so far within round $t$. This is outlined in the \texttt{Update-Procedure} subroutine of Algorithm~\ref{alg:maml}. Note that this is different from the update $\mathbf{U}_t$ used within the meta-optimization, which uses a fixed-size minibatch since many-shot meta-learning is computationally expensive and memory intensive. In practice, we meta-train with update minibatches of size at-most $25$, whereas evaluation may use hundreds of datapoints for some tasks. After the model is updated, we measure the performance using held-out data $\data_t^\text{test}$ from task $\task_t$. This data is not revealed to the online meta-learner at any time. Further, we also evaluate task learning efficiency, which corresponds to  the size of $\data_t$ required to achieve a specified performance threshold $\gamma$, e.g. $\gamma=90\%$ classification accuracy or $\gamma$ corresponds to a certain loss value. If less data is sufficient to reach the threshold, then priors learned from previous tasks are being useful and we have achieved positive transfer.
\section{Experimental Evaluation}
\label{sec:experiments}

Our experimental evaluation studies the practical FTML algorithm (Section~\ref{sec:real_algorithm}) in the context of vision-based online learning problems. These problems include synthetic modifications of the MNIST dataset, pose detection with synthetic images based on PASCAL3D+ models~\cite{xiang2014beyond}, and realistic online image classification experiments with the CIFAR-100 dataset.
The aim of our experimental evaluation is to study the following questions:
(1) can online meta-learning (and specifically FTML) be successfully applied to multiple non-stationary learning problems? and
(2) does online meta-learning (FTML) provide empirical benefits over prior methods?

To this end, we compare to the following algorithms:
(a)~Train on everything (TOE) trains on all available data so far (including $\data_t$ at round $t$) and trains a single predictive model. This model is directly tested without any specific adaptation since it has already been trained on $\data_t$. (b) Train from scratch, which initializes $\param_t$ randomly, and finetunes it using $\data_t$. (c) Joint training with fine-tuning, which at round $t$, trains on all the data jointly till round $t-1$, and then finetunes it specifically to round $t$ using only $\data_t$. This corresponds to the standard online learning approach where FTL is used (without any meta-learning objective), followed by task-specific fine-tuning.

We note that TOE is a very strong point of comparison, capable of reusing representations across tasks, as has been proposed in a number of prior continual learning works~\cite{rusu2016progressive,aljundi2017expert,wang2017growing}.
However, unlike FTML, TOE does not explicitly learn the structure across tasks. Thus, it may not be able to fully utilize the information present in the data, and will likely not be able to learn new tasks with only a few examples.
Further, the model might incur negative transfer if the new task differs substantially from previously seen ones, as has been observed in prior work~\cite{actormimic}. Training on each task independently from scratch avoids negative transfer, but also precludes any reuse between tasks. When the amount of data for a given task is large, we may expect training from scratch to perform well since it can avoid negative transfer and can learn specifically for the particular task. Finally, FTL with fine-tuning represents a natural online learning comparison, which in principle should combine the best parts of learning from scratch and TOE, since this approach adapts specifically to each task \emph{and} benefits from prior data. However, in contrast to FTML, this method does not explicitly meta-learn and hence may not fully utilize any structure in the tasks.

\begin{figure*}[t]
\setlength{\unitlength}{0.5\columnwidth}
\begin{picture}(1.99,0.8) \linethickness{0.5pt}
\put(0.04,0.49){\includegraphics[height=0.08\linewidth]{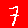}}
\put(0.37,0.49){\includegraphics[height=0.08\linewidth]{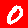}}
\put(0.70,0.49){\includegraphics[height=0.08\linewidth]{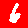}}
\put(1.03,0.49){\includegraphics[height=0.08\linewidth]{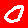}}
\put(1.41,0.49){\includegraphics[height=0.08\linewidth]{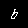}}
\put(1.74,0.49){\includegraphics[height=0.08\linewidth]{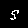}}
\put(2.07,0.49){\includegraphics[height=0.08\linewidth]{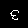}}
\put(2.40,0.49){\includegraphics[height=0.08\linewidth]{image0_bk1.png}}
\put(2.78,0.49){\includegraphics[height=0.08\linewidth]{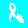}}
\put(3.11,0.49){\includegraphics[height=0.08\linewidth]{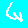}}
\put(3.44,0.49){\includegraphics[height=0.08\linewidth]{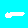}}
\put(3.77,0.49){\includegraphics[height=0.08\linewidth]{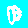}}

\put(0.04,-0.0){\includegraphics[height=0.107\linewidth]{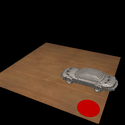}}
\put(0.48,-0.0){\includegraphics[height=0.107\linewidth]{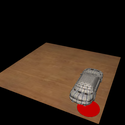}}
\put(0.92,-0.0){\includegraphics[height=0.107\linewidth]{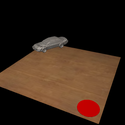}}
\put(1.41,-0.0){\includegraphics[height=0.107\linewidth]{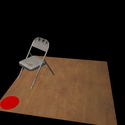}}
\put(1.85,-0.0){\includegraphics[height=0.107\linewidth]{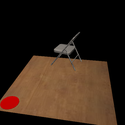}}
\put(2.29,-0.0){\includegraphics[height=0.107\linewidth]{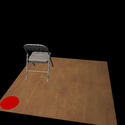}}
\put(2.78,-0.0){\includegraphics[height=0.107\linewidth]{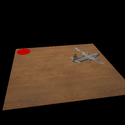}}
\put(3.22,-0.0){\includegraphics[height=0.107\linewidth]{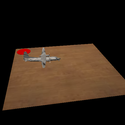}}
\put(3.66,-0.0){\includegraphics[height=0.107\linewidth]{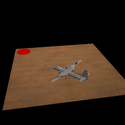}}

\end{picture}
\vspace{-0.3cm}
\caption{\small Illustration of three tasks for Rainbow MNIST (top) and pose prediction (bottom). CIFAR images not shown. Rainbow MNIST includes different rotations, scaling factors, and background colors.
For the pose prediction tasks, the goal is to predict the global position and orientation of the object on the table. Cross-task variation includes varying 50 different object models within 9 object classes, varying object scales, and different camera viewpoints.
\label{fig:tasks}
}
\end{figure*}

\begin{figure*}[t]
\setlength{\unitlength}{0.5\columnwidth}
\begin{picture}(1.99,1.1) \linethickness{0.5pt}
\put(0.0,0.0){\includegraphics[width=0.375\linewidth]{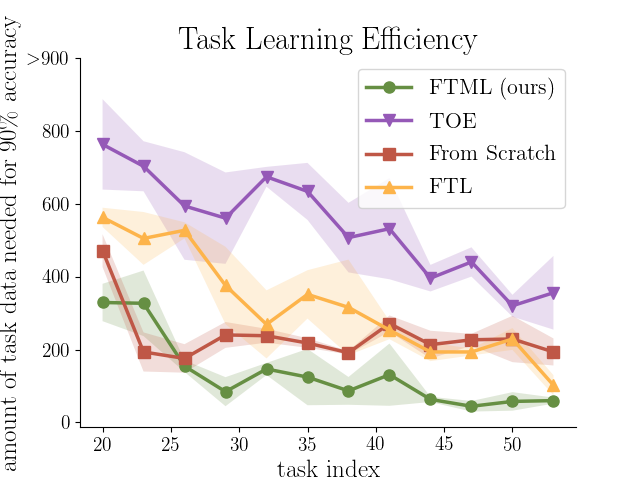}}
\put(1.38,0.0){\includegraphics[width=0.375\linewidth]{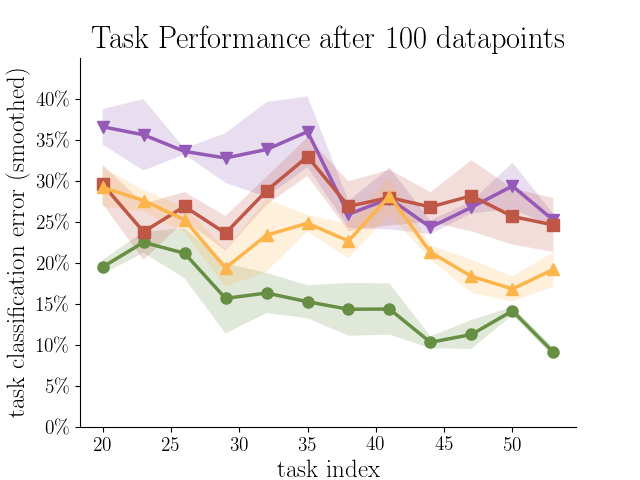}}
\put(2.75,0.0){\includegraphics[width=0.375\linewidth]{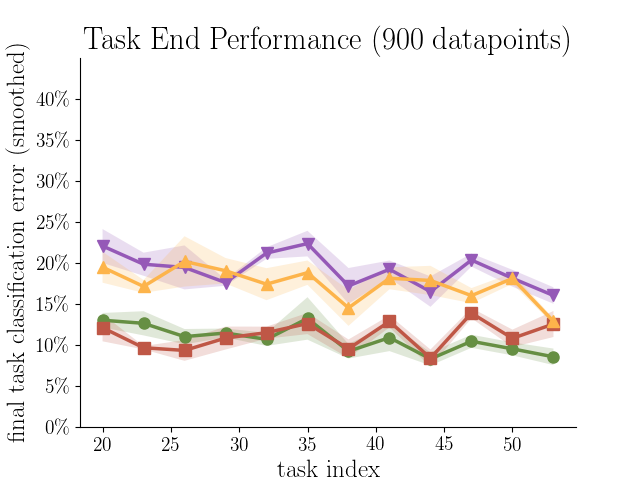}}
\end{picture}
\vspace{-0.3cm}
\caption{\small Rainbow MNIST results. Left: amount of data needed to learn each new task. Center: task performance after 100 datapoints on the current task. Right: The task performance after all 900 datapoints for the current task have been received. Lower is better for all plots, while shaded regions show standard error computed using three random seeds. FTML can learn new tasks more and more efficiently as each new task is received, demonstrating effective forward transfer.
\label{fig:rainbow}
}
\end{figure*}

\subsection{Rainbow MNIST}

In this experiment, we create a sequence of tasks based on the MNIST character recognition dataset. We transform the digits in a number of ways to create different tasks, such as 7 different colored backgrounds, 2 scales (half size and original size), and 4 rotations of 90 degree intervals. As illustrated in Figure~\ref{fig:tasks}, a task involves correctly classifying digits with a randomly sampled background, scale, and rotation. This leads to 56 total tasks. We partitioned the MNIST training dataset into 56 batches of examples, each with 900 images and applied the corresponding task transformation to each batch of images. The ordering of tasks was selected at random and we set $90\%$ classification accuracy as the proficiency threshold.

The learning curves in Figure~\ref{fig:rainbow} show that FTML learns tasks more and more quickly, with each new task added. We also observe that FTML substantially outperforms the alternative approaches in both efficiency and final performance. FTL performance better than TOE since it performs task-specific adaptation, but its performance is still inferior to FTML.
We hypothesize that, while the prior methods improve in efficiency over the course of learning as they see more tasks, they struggle to prevent negative transfer on each new task. Our last observation is that training independent models does not learn efficiently, compared to models that incorporate data from other tasks; but, their final performance with $900$ data points is similar.

\subsection{Five-Way CIFAR-100}

\begin{figure*}[t]
\setlength{\unitlength}{0.25\linewidth}
\begin{picture}(1.99,1.1) \linethickness{0.5pt}
\vspace{-0.2cm}
\put(0.0,0.0){\includegraphics[width=0.375\linewidth]{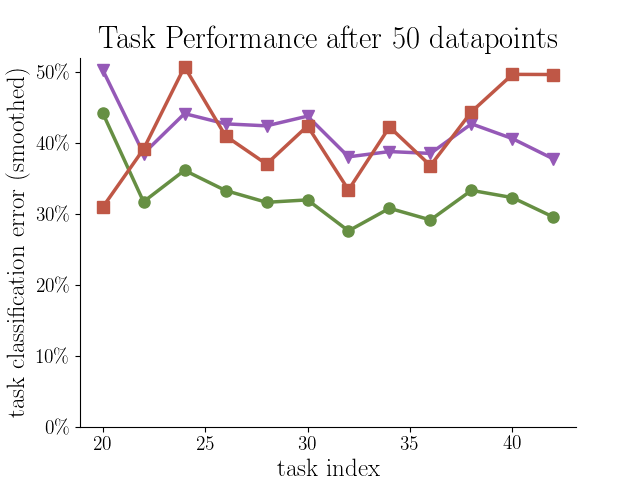}}
\put(1.33,0.0){\includegraphics[width=0.375\linewidth]{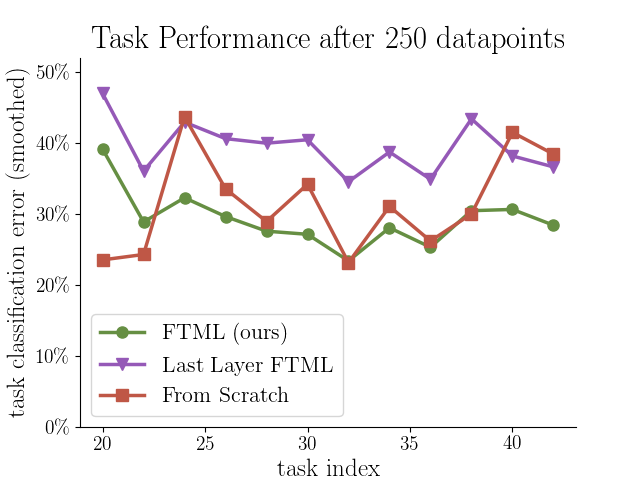}}
\put(2.675,0.0){\includegraphics[width=0.375\linewidth]{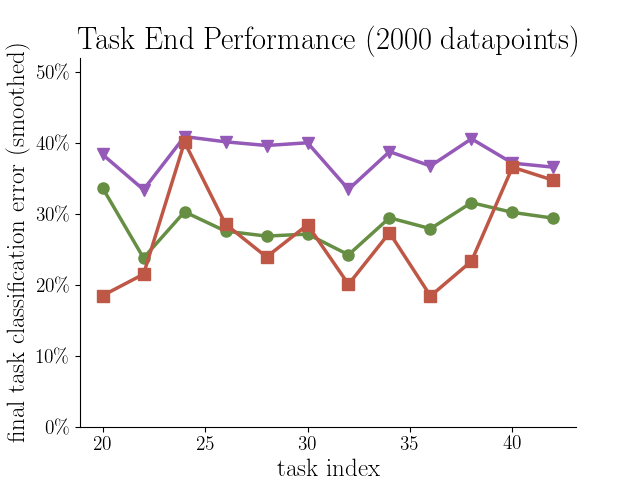}}
\end{picture}
\vspace{-0.3cm}
\caption{\small Online CIFAR-100 results, evaluating task performance after 50, 250, and 2000 datapoints have been received for a given task. We see that FTML learns each task much more efficiently than models trained from scratch, while both achieve similar asymptotic performance after 2000 datapoints. We also observe that FTML benefits from adapting all layers rather than learning a shared feature space across tasks while adapting only the last layer.
\label{fig:cifar}
}
\end{figure*}

\begin{figure*}[t]
\setlength{\unitlength}{0.25\linewidth}
\begin{picture}(1.99,1.1) \linethickness{0.5pt}
\put(0.0,0.0){\includegraphics[width=0.375\linewidth]{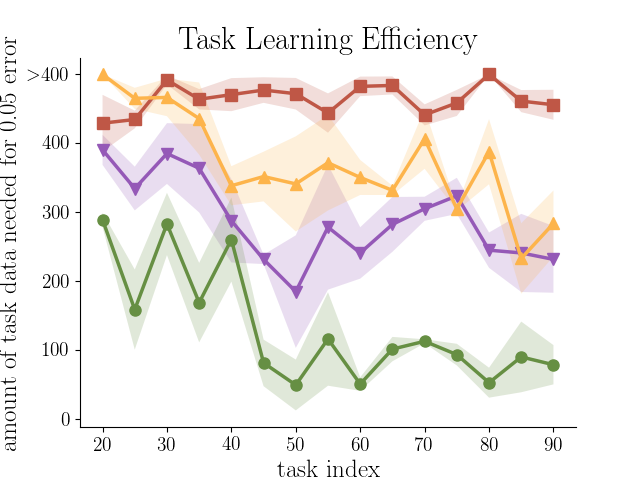}}
\put(1.325,0.0){\includegraphics[width=0.375\linewidth]{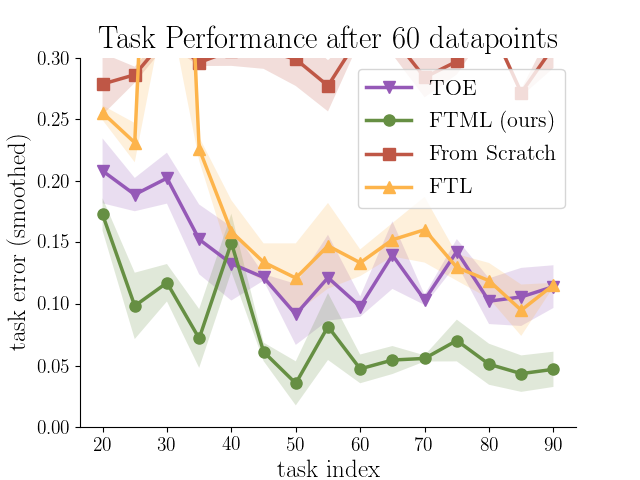}}
\put(2.655,0.0){\includegraphics[width=0.375\linewidth]{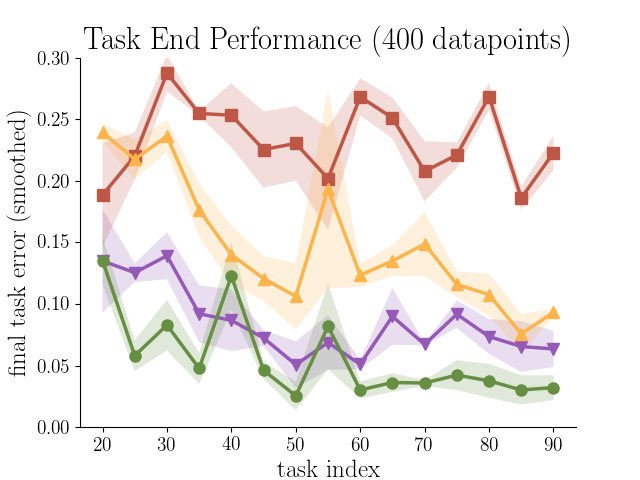}}
\end{picture}
\vspace{-0.3cm}
\caption{\small Object pose prediction results. Left: we observe that online meta-learning generally leads to faster learning as more and more tasks are introduced, learning with only $10$ datapoints for many of the tasks. Center \& right, we see that meta-learning enables transfer not just for faster learning but also for more effective performance when $60$ and $400$ datapoints of each task are available. The order of tasks is randomized, leading to spikes when more difficult tasks are introduced. Shaded regions show standard error across three random seeds 
\label{fig:pascal}
}
\end{figure*}

In this experiment, we create a sequence of 5-way classification tasks based on the CIFAR-100 dataset, which contains more challenging and realistic RGB images than MNIST. Each classification problem involves a newly-introduced class from the 100 classes in CIFAR-100.
Thus, different tasks  correspond to different labels spaces.
The ordering of tasks is selected at random, and we measure performance using classification accuracy. Since it is less clear what the proficiency threshold should be for this task, we evaluate the accuracy on each task after varying numbers of datapoints have been seen. Since these tasks are mutually exclusive (as label space is changing), it makes sense to train the TOE model with a different final layer for each task. An extremely similar approach to this is to use our meta-learning approach but to only allow the final layer parameters to be adapted to each task. Further, such a meta-learning approach is a more direct comparison to our full FTML method, and the comparison can provide insight into whether online meta-learning is simply learning features and performing training on the last layer, or if it is adapting the features to each task. Thus, we  compare to this last layer online meta-learning approach instead of TOE with multiple heads.
The results (see Figure~\ref{fig:cifar}) indicate that FTML learns more efficiently than independent models and a model with a shared feature space. The results on the right indicate that training from scratch achieves good performance with $2000$ datapoints, reaching similar performance to FTML. However, the last layer variant of FTML seems to not have the capacity to reach good performance on all tasks.

\subsection{Sequential Object Pose Prediction}

In our final experiment, we study a 3D pose prediction problem. Each task involves learning to predict the global position and orientation of an object in an image. We construct a dataset of synthetic images using $50$ object models from 9 different object classes in the PASCAL3D+ dataset~\cite{xiang2014beyond}, rendering the objects on a table using the renderer accompanying the MuJoCo~\cite{mujoco} (see Figure~\ref{fig:tasks}). To place an object on the table, we select a random 2D location, as well as a random azimuthal angle. Each task corresponds to a different object with a randomly sampled camera angle. 
We place a red dot on one corner of the table to provide a global reference point for the position.
Using this setup, we construct 90 tasks (with an average of about 2 camera viewpoints per object), with 1000 datapoints per task.
All models are trained to regress to the global 2D position and the sine and cosine of the azimuthal angle (the angle of rotation along the z-axis).
For the loss functions, we use mean-squared error, and set the proficiency threshold to an error of $0.05$. 
We show the results of this experiment in Figure~\ref{fig:pascal}. The results demonstrate that meta-learning can improve both efficiency and performance of new tasks over the course of learning, solving many of the tasks with only $10$ datapoints. Unlike the previous settings, TOE substantially outperforms training from scratch, indicating that it can effectively make use of the previous data from other tasks, likely due to the greater structural similarity between the pose detection tasks. 
However, the superior performance of FTML suggests that even better transfer can be accomplished through meta-learning. Finally, we find that FTL performs comparably or worse than TOE, indicating that task-specific fine-tuning can lead to overfitting when the model is not explicitly trained for the ability to be fine-tuned effectively.
\section{Connections to Related Work}
\label{sec:related_works}

Our work proposes to use meta-learning or learning to learn~\cite{thrun,schmidhuber1987,naik}, in the context of online (regret-based) learning. We reviewed the foundations of these approaches in Section~\ref{sec:foundations}, and we summarize additional related work along different axis.

\textbf{Meta-learning:} 
Prior works have proposed learning update rules, selective copying of weights, or optimizers for fast adaptation~\cite{hochreiter,bengiobengio1,learntolearnbygdbygd,learntooptimize,hugo,Schmidhuber2002OptimalOP, hypernets}, as well as recurrent models that learn by ingesting datasets directly~\cite{mann,rl2,learningrl,metanetworks,mishra2017meta}. 
While some 
meta-learning works have considered online learning settings at \emph{meta-test time}~\cite{mann,AlShedivat2017ContinuousAV,nagabandi2018deep}, nearly all prior meta-learning algorithms assume that the \emph{meta-training tasks} come from a stationary distribution. Furthermore, most prior work has not evaluated versions of meta-learning algorithms when presented with a continuous stream of tasks. One exception is work that adapts hyperparameters online~\cite{elfwing2017online,meier2017online,baydin2017online}. In contrast, we consider a more flexible approach that allows for adapting all of the parameters in the model for each task. More recent work has considered handling non-stationary task distributions in meta-learning using Dirichlet process mixture models over meta-learned parameters~\cite{erin_openreview}. Unlike this prior work, we introduce a simple extension onto the MAML algorithm without mixtures over parameters, and provide theoretical guarantees.

\textbf{Continual learning:}
Our problem setting is related to (but distinct from) continual, or lifelong learning~\cite{thrun1998lifelong,zhao1996incremental}. In lifelong learning, a number of recent papers have focused on avoiding forgetting and negative backward transfer~\cite{goodfellow2013empirical,kirkpatrick2017overcoming,zenke2017continual,rebuffi2017icarl,shin2017continual,shmelkov2017incremental,gradient_episodic,nguyen2017variational,Schmidhuber2013PowerPlayTA}. Other papers have focused on maintaining a reasonable model capacity as new tasks are added~\cite{lee2017lifelong,mallya2017packnet}. In this paper, we sidestep the problem of catastrophic forgetting by maintaining a buffer of all the observed data~\cite{isele2018selective}. In future work, we hope to understand the interplay between limited memory and catastrophic forgetting for variants of the FTML algorithm.
Here, we instead focuses on the problem of forward transfer: maximizing the efficiency of learning new tasks within a non-stationary learning setting. 
Prior works have also considered settings that combine joint training across tasks (in a sequential setting) with task-specific adaptation~\cite{Barto1995, Lowrey-ICLR-19}, but have not explicitly employed meta-learning.
Furthermore,
unlike prior works~\cite{ruvolo2013ella,rusu2016progressive,aljundi2017expert,wang2017growing}, we also focus on the setting where there are several tens or hundreds of tasks. This setting is interesting since there is significantly more information that can be transferred from previous tasks and we can employ more sophisticated techniques such as meta-learning for transfer, enabling the agent to move towards few-shot learning after experiencing a large number of tasks. 

\textbf{Online learning:} Similar to continual learning, online learning deals with a sequential setting with streaming tasks. It is well known in online learning that FTL is computationally expensive, with a large body of work dedicated to developing cheaper algorithms~\cite{CesaBianchi2006PLA, Hazan2006LogarithmicRA, Zinkevich2003OnlineCP, ShaiBook}. 
Again, in this work, we sidestep the computational considerations to first study if the meta-learning analog of FTL can provide performance gains. For this, we derived the FTML algorithm which has low regret when compared to a powerful adaptive comparator class that performs task-specific adaptation. We leave the design of more computationally efficient versions of FTML to future work.

To avoid the pitfalls associated with a single best model in hindsight, online learning literature has also studied alternate notions of regret, with the closest settings being dynamic regret and adaptive or tracking regret. In the dynamic regret setting~\citep{Herbster1995TrackingTB, Yang2016TrackingSM, Besbes2015NonstationarySO}, the performance of the online learner's model sequence is compared against the sequence of optimal solutions corresponding to each loss function in the sequence. Unfortunately, lower-bounds~\citep{Yang2016TrackingSM} suggest that the comparator class is too powerful and may not provide for any non-trivial learning in the general case. To overcome these limitations, prior work has placed restrictions on how quickly the loss functions or the comparator model can change~\citep{Hazan2009EfficientLA, Hall2015OnlineCO, Herbster1995TrackingTB}. In contrast, we consider a different notion of adaptive regret, where the learner and comparator both have access to an update procedure. The update procedures allow the comparator to produce different models for different loss functions, thereby serving as a powerful comparator class (in comparison to a fixed model in hindsight). For this setting, we derived sublinear regret algorithms without placing any restrictions on the sequence of loss functions.
Recent and concurrent works have also studied algorithms related to MAML and its first order variants using theoretical tools from online learning literature~\cite{Alquier2016RegretBF, Denevi2019LearningtoLearnSG, Khodak2019ProvableGF}. These works also derive regret and generalization bounds, but these algorithms have not yet been empirically studied in large scale domains or in non-stationary settings. 
We believe that our online meta-learning setting captures the spirit and practice of continual lifelong learning, and also shows promising empirical results.

\section{Discussion and Future Work}

In this paper, we introduced the online meta-learning problem statement, with the aim of connecting the fields of meta-learning and online learning. Online meta-learning provides, in some sense, a more natural perspective on the ideal real-world learning procedure: an intelligent agent interacting with a constantly changing environment should utilize streaming experience to both master the task at hand, and become more proficient at learning new tasks in the future. We analyzed the proposed FTML algorithm and showed that it achieves logarithmic regret. We then illustrated how FTML can be adapted to a practical algorithm. Our experimental evaluations demonstrated that the proposed algorithm outperforms prior methods.
In the rest of this section, we reiterate a few salient features of the online meta learning setting (Section~\ref{sec:problem_setting}), and outline avenues for future work.

\textbf{More powerful update procedures.} 
In this work, we concentrated our analysis on the case where the update procedure $\ud_t$, inspired by MAML, corresponds to one step of gradient descent. However, in practice, many works with MAML (including our experimental evaluation) use multiple gradient steps in the update procedure, and back-propagate through the entire path taken by these multiple gradient steps. Analyzing this case, and potentially higher order update rules will also make for exciting future work.

\textbf{Memory and computational constraints.} 
In this work, we primarily aimed to discern if it is possible to meta-learn in a sequential setting. For this purpose, we proposed the FTML template algorithm which draws inspiration from FTL in online learning. As discussed in Section~\ref{sec:related_works}, it is well known that FTL has poor computational properties, since the computational cost of FTL grows over time as new loss functions are accumulated. Further, in many practical online learning problems, it is challenging (and sometimes impossible) to store all datapoints from previous tasks. While we showed that our method can effectively learn nearly 100 tasks in sequence without significant burdens on compute or memory, scalability remains a concern. Can a more streaming algorithm like mirror descent that does not store all the past experiences be successful as well? Our main theoretical results (Section~\ref{sec:grad_step_analysis}) suggests that there exist a large family of online meta-learning algorithms that enjoy sublinear regret. Tapping into the large body of work in online learning, particularly mirror descent, to develop computationally cheaper algorithms would make for exciting future work.

\section*{Acknowledgements}
Aravind Rajeswaran thanks Emo Todorov for valuable discussions on the problem formulation. This work was supported by the National Science Foundation via IIS-1651843, Google, Amazon, and NVIDIA. Sham Kakade acknowledges funding from the Washington Research Foundation Fund for Innovation in Data-Intensive Discovery and the NSF CCF 1740551 award.

\bibliography{citations}
\bibliographystyle{icml2019}

\appendix
\onecolumn
\clearpage

\section{Linear Regression Example}
\label{app:quadratic_example}

Here, we present a simple example of optimizing a collection of quadratic objectives (equivalent to linear regression on fixed set of features), where the solutions to joint training and the meta-learning (MAML) problem are different. The purpose of this example is to primarily illustrate that meta-learning can provide performance gains even in seemingly simple and restrictive settings. Consider a collection of objective functions: $\{ \fn_i : \param \in \bR^d \rightarrow \bR \}_{i=1}^M$ which can be described by quadratic forms. Specifically, each of these functions are of then form
\[
\fn_i(\param) = \frac{1}{2} \param^T \Aq_i \param + \param^T \bq_i.
\]
This can represent linear regression problems as follows: let $(\inp_{\task_i}, \out_{\task_i})$ represent input-output pairs corresponding to task $\task_i$. Let the predictive model be $\bm{h}(\inp) = \param^T \inp$. Here, we assume that a constant scalar (say $1$) is concatenated in $\inp$ to subsume the constant offset term (as common in practice). Then, the loss function can be written as:
\[
\fn_i(\param) = \frac{1}{2} \bE_{(\inp, \out) \sim \task_i} \left[ || \bm{h}(\inp) - \out ||^2  \right]
\]
which corresponds to having $\Aq_i = \bE_{\inp \sim \task_i}[\inp \inp^T]$ and  \hbox{$\bq_i = \bE_{(\inp, \out) \sim \task_i}[\inp^T \out]$.} For these set of problems, we are interested in studying the difference between joint training and meta-learning.

\paragraph{Joint training} The first approach of interest is joint training which corresponds to the optimization problem
\begin{equation} \label{eq:quadratic_joint}
    \min_{\param \in \bR^d} F(\param) \, , \textrm{ where } F(\param) = \frac{1}{M} \sum_{i=1}^M \fn_i(\param).
\end{equation}
Using the form of $\fn_i$, we have
\[
F(\param) = \frac{1}{2} \ \param^T \left(  \frac{1}{M} \sum_{i=1}^M \Aq_i \right) \param + \param^T \left( \frac{1}{M} \sum_{i=1}^M \bq_i \right).
\]
Let us define the following:
\[
\bar{\Aq} := \frac{1}{M} \sum_{i=1}^M \Aq_i \ \text{ and } \ \bar{\bq} := \frac{1}{M} \sum_{i=1}^M \bq_i.
\]
The solution to the joint training optimization problem (Eq.~\ref{eq:quadratic_joint}) is then given by $\param_{\text{joint}}^* = - \bar{\Aq}^{-1} \bar{\bq}$.

\paragraph{Meta learning (MAML)} The second approach of interest is meta-learning, which as mentioned in Section~\ref{sec:maml_foundations} corresponds to the optimization problem:
\begin{equation} \label{eq:quadratic_maml}
    \min_{\param \in \bR^d} \tilde{F}(\param) \, , \textrm{ where } \tilde{F}(\param) = \frac{1}{M} \sum_{i=1}^M \fn_i(\ud_i(\param)).
\end{equation}
Here, we specifically concentrate on the 1-step (exact) gradient update procedure: $\ud_i(\param) = \param - \alpha \nabla \fn_i(\param)$. In the case of the quadratic objectives, this leads to:
\begin{equation*}
    \begin{split}
    \fn_i(\ud_i(\param)) &= \frac{1}{2}  (\param - \alpha \Aq_i \param - \alpha \bq_i)^T \Aq_i (\param - \alpha \Aq_i \param - \alpha \bq_i) \\
    & + (\param - \alpha \Aq_i \param - \alpha \bq_i)^T \bq_i
    \end{split}
\end{equation*}
The corresponding gradient can be written as:
\begin{equation*}
    \begin{split}
    \nabla \fn_i(\ud_i(\param)) &= \bigg( \eye - \alpha \Aq_i \bigg) \bigg( \Aq_i \big( \param - \alpha \Aq_i \param - \alpha \bq_i \big) + \bq_i \bigg) \\
    &= \big( \eye - \alpha \Aq_i \big) \Aq_i \big( \eye - \alpha \Aq_i \big) \param + \big( \eye - \alpha \Aq_i \big)^2 \bq_i
    \end{split}
\end{equation*}
For notational convenience, we define:
\begin{equation*}
    \begin{split}
    \Aq_\dagger & := \frac{1}{M} \sum_{i=1}^M \big( \eye - \alpha \Aq_i \big)^2 \Aq_i \\
    \bq_\dagger & := \frac{1}{M} \sum_{i=1}^M \big( \eye - \alpha \Aq_i \big)^2 \bq_i .
    \end{split}
\end{equation*}
Then, the solution to the MAML optimization problem (Eq.~\ref{eq:quadratic_maml}) is given by $\param_{\text{MAML}}^* = - \Aq_\dagger^{-1} \bq_\dagger$. 

\paragraph{Remarks} 
In general, \hbox{$\param_{\text{joint}}^* \neq \param_{\text{MAML}}^*$} based on our analysis. Note that $\Aq_\dagger$ is a weighed average of different $\Aq_i$, but the weights themselves are a function of $\Aq_i$. The reason for the difference between $\param_{\text{joint}}^*$ and $\param_{\text{MAML}}^*$ is the difference in moments of input distributions. The two solutions, $\param_{\text{joint}}^*$ and $\param_{\text{MAML}}^*$, coincide when $\Aq_i = \Aq \ \forall i$.  Furthermore, since $\param_{\text{MAML}}^*$ was optimized to explicitly minimize $\tilde{F}(\cdot)$, it would lead to better performance after task-specific adaptation. 

This example and analysis reveals that there is a clear separation in performance between joint training and meta-learning even in the case of quadratic loss functions. Improved performance with meta-learning approaches have been noted empirically with non-convex loss landscapes induced by neural networks. Our example illustrates that meta learning can provide non-trivial gains over joint training even in simple convex loss landscapes.

\clearpage
\section{Theoretical Analysis}
\label{app:proofs}
We first begin by reviewing some definitions and proving some helper lemmas before proving the main theorem.

\subsection{Notations, Definitions, and Some Properties} 
\begin{itemize}
    \item For a vector $\px \in \bR^d$, we use $\| \px \|$ to denote the $L2$ norm, i.e. $\sqrt{\px^T \px}$.
    \item For a matrix $\bm{A} \in \bR^{m \times d}$, we use $\| \bm{A} \|$ to denote the matrix norm induced by the vector norm,
    \[
    \| \bm{A} \| = \mathrm{sup} \bigg\{ \frac{\| \bm{A} \px \|}{\| \px \|} : \px \in \bR^d \ \mathrm{with} \ \| \px \| \neq 0  \bigg\}
    \]
    We briefly review two properties of this norm that are useful for us
    \begin{itemize}
        \item $\| \bm{A} + \bm{B} \| \leq \| \bm{A} \| + \| \bm{B} \|$ (triangle inequality)
        \item $\| \bm{A} \px \| \leq \| \bm{A} \| \| \px \|$ (sub-multiplicative property)
    \end{itemize}
    Both can be seen easily by recognizing that an alternate equivalent definition of the induced norm is $\| \bm{A} \| \geq \frac{\| \bm{A} \px \|}{\| \px \|} \ \forall \px$.
    \item For (symmetric) matrices $\bm{A} \in \bR^{d \times d}$, we use $\bm{A} \succeq \zero$ to denote the positive semi-definite nature of the matrix. Similarly, $\bm{A} \succeq \bm{B}$ implies that $\px^T (\bm{A} - \bm{B}) \px \geq 0 \ \forall \px $.
\end{itemize} 

\begin{definition}
A function $f(\cdot)$ is $G-${\bf Lipschitz continuous} iff:
\[
\| \fn(\px) - \fn(\py) \| \leq G \| \px - \py \| \hspace*{10pt} \forall \ \px, \py
\]
\end{definition}

\begin{definition}
A differentiable function $\fn(\cdot)$ is $\beta-${\bf smooth} (or $\beta-${\bf gradient Lipschitz}) iff:
\[
\| \nabla \fn(\px) - \nabla \fn(\py) \| \leq \beta \| \px - \py \| \hspace*{10pt} \forall \ \px, \py
\]
\end{definition}

\begin{definition}
A twice-differentiable function $\fn(\cdot)$ is $\rho-${\bf Hessian Lipschitz} iff:
\[
\| \nabla^2 \fn(\px) - \nabla^2 \fn(\py) \| \leq \rho \| \px - \py \| \hspace*{10pt} \forall \  \px, \py
\]
\end{definition}

\begin{definition}
A twice-differentiable function $\fn(\cdot)$ is $\mu-${\bf strongly convex} iff:
\[
\nabla^2 \fn(\px) \succeq \mu \eye \hspace*{10pt} \forall \ \px
\]
\end{definition}

We also review some properties of convex functions that help with our analysis.
\begin{lemma}
(Elementery properties of convex functions; see e.g. \citet{BoydBook,ShaiBook})
\begin{enumerate}
    \itemsep0em
    \item If $\fn(\cdot)$ is differentiable, convex, and $G-$Lipschitz continuous, then 
    $$ \| \nabla \fn(\px) \| \leq G \ \ \forall \ \px.$$
    \item If $\fn(\cdot)$ is twice-differentiable, convex, and $\beta-$smooth, then
    $$ \nabla^2 \fn(\px) \preceq \beta \eye \ \ \forall \ \px. $$
    This is also equivalent to $\| \nabla^2 \fn(\px) \| \leq \beta \ \ \forall \ \px$.
    The strong convexity also implies that
    \[
    \| \fn(\px) - \fn(\py) \| \geq \mu \| \px-\py \| \hspace*{10pt} \forall \ \px, \py.
    \]
\end{enumerate}
\end{lemma}

\subsection{Descent Lemma}
In this sub-section, we prove a useful lemma about the dynamics of gradient descent and its contractive property. For this, we first consider a lemma that provides an inequality analogous to that of mean-value theorem.

\begin{lemma}
\label{lemma:line_integral}
(Fundamental theorem of line integrals)
Let $\varphi : \bR^d \rightarrow \bR$ be a differentiable function. Let $\px, \py \in \bR^d$ be two arbitrary points, and let $\bm{r}(\gamma)$ be a curve from $\px$ to $\py$, parameterized by scalar $\gamma$. Then,
\[
\varphi(\py) - \varphi(\px) = \int_{\gamma[\px, \py]} \nabla \varphi(\bm{r}(\gamma)) d \bm{r} = \int_{\gamma[\px, \py]} \nabla \varphi(\bm{r}(\gamma)) \bm{r}'(\gamma) d\gamma
\]
\end{lemma}

\begin{lemma}
\label{lemma:mean_value_ineq}
(Mean value inequality)
Let $\bfn : \px \in \bR^d \rightarrow \bR^m$ be a differentiable function. Let $\jacobian\bfn$ be the function Jacobian, such that $\jacobian\bfn_{i,j} = \frac{\partial \varphi_i}{\partial \px_j}$. Let $M = \max_{\px \in \bR^d} \| \jacobian\bfn(\px) \|$. Then, we have
\[
\| \bfn(\px) - \bfn(\py) \| \leq M \ \| \px - \py \| \hspace*{10pt} \forall \ \px, \py \in \bR^d
\]
\end{lemma}
\begin{proof}
Let the line segment connecting $\px$ and $\py$ be parameterized as $\pz(t) = \py + t (\px-\py), \ t \in [0,1]$. Using Lemma~\ref{lemma:line_integral} for each coordinate, we have that
\[
\bfn(\px) - \bfn(\py) = \left( \int_{t=0}^{t=1} \jacobian\bfn \big(\pz(t)\big) dt \right) \bigg( \pz(1) - \pz(0) \bigg) = \left( \int_{t=0}^{t=1} \jacobian\bfn \big(\pz(t)\big) dt \right) \bigg( \px - \py \bigg)
\]
\begin{align*}
    \| \bfn(\px) - \bfn(\py) \| & = \left\| \int_{t=0}^{t=1} \jacobian\bfn \big(\pz(t)\big) \big( \px-\py \big) dt \right\| \\
    & \overset{(a)}{\leq} \int_{t=0}^{t=1} \left\| \jacobian\bfn \big(\pz(t)\big) \big( \px-\py \big) \right\| dt \\
    & \overset{(b)}{\leq} \int_{t=0}^{t=1} \left\| \jacobian\bfn \big(\pz(t)\big) \right\| \left\| \big( \px-\py \big) \right\| dt \\
    & \overset{(c)}{\leq} \int_{t=0}^{t=1} M \left\| \px-\py \right\| dt \\
    & = M \| \px - \py \| \int_{t=0}^{t=1} dt = M \|\px - \py \|
\end{align*}
Step (a) follows from the Cauchy-Schwartz inequality. Step (b) is a consequence of the sub-multiplicative property. Finally, step (c) is based on the definition of $M$.
\end{proof}

\begin{lemma}
\label{lemma:gd_contraction}
(Contractive property of gradient descent)
Let $\varphi : \bR^d \rightarrow \bR$ be an arbitrary function that is $G-$Lipschitz, $\beta-$smooth, and $\mu-$strongly convex. Let $\ud(\px) = \px - \alpha \nabla \varphi(\px)$ be the gradient descent update rule with $\alpha \leq \frac{1}{\beta}$. Then,
\[
\| \ud(\px) - \ud(\py) \|  \leq (1-\alpha \mu) \| \px - \py \| \hspace*{5pt} \forall \ \px, \py \in \bR^d.
\]
\end{lemma}
\begin{proof}
Firstly, the Jacobian of $\ud(\cdot)$ is given by $ \jacobian\ud (\px) = \eye - \alpha \nabla^2 \varphi(\px) $. Since $\mu \eye \preceq \nabla^2 \varphi(\px) \preceq \beta \eye \ \forall \px \in \bR^d$, we can bound the Jacobian as
\[
(1-\alpha \beta) \eye \preceq \jacobian \ud(\px) \preceq (1-\alpha \mu) \eye \hspace*{10pt} \forall \ \px \in \bR^d
\]
The upper bound also implies that $\| \jacobian\ud(\px) \| \leq (1-\alpha \mu) \ \forall \px$. Using this and invoking Lemma~\ref{lemma:mean_value_ineq} on $\ud(\cdot)$, we get that
\[
\| \ud(\px) - \ud(\py) \| \leq (1-\alpha \mu) \| \px-\py \|.
\]
\end{proof}

\subsection{Main Theorem}
Using the previous lemmas, we can prove our main theorem. We restate the assumptions and statement of the theorem and then present the proof.

\noindent \textbf{Assumption 1.}
\textit{
($C^2$-smoothness) Suppose that $\fn(\cdot)$ is twice differentiable and
\begin{itemize}
    \item $\fn(\cdot)$ is $G-$Lipschitz in function value.
    \item $\fn(\cdot)$ is $\beta-$smooth, or has $\beta-$Lipschitz gradients.
    \item $\fn(\cdot)$ has $\rho-$Lipschitz hessian.
\end{itemize}
}

\noindent \textbf{Assumption 2.}
\textit{(Strong convexity)} Suppose that $\fn(\cdot)$ is convex. Further suppose that it is $\mu-$strongly convex.
 
\noindent \textbf{Theorem 1.}
\textit{
Suppose $\fn$ and $\fnht: \bR^d \rightarrow \bR$ satisfy assumptions 1 and 2. Let $\udfn$ be the function evaluated after a one step gradient update procedure, i.e.
}
$$
\udfn(\param) \coloneqq \fn \big( \param - \alpha \nabla \fnht(\param) \big).
$$
\textit{If the step size is selected as $\alpha \leq \min \{ \frac{1}{2\beta}, \frac{\mu}{8 \rho G} \}$, then $\udfn$ is convex. Furthermore, it is also $\tilde{\beta}=9\beta/8$ smooth and $\tilde{\mu} = \mu/8$ strongly convex.}

\begin{proof}
Consider two arbitrary points $\px, \py \in \bR^d$. Using the chain rule and our shorthand of $\udpx \equiv \ud(\px), \udpy \equiv \ud(\px)$, we have
\begin{align*}
    \nabla \udfn(\px) - \nabla \udfn(\py) & = \nabla \ud(\px) \nabla \fn(\udpx) - \nabla \ud(\py) \nabla \fn(\udpy) \\
    & = \left( \nabla \ud(\px) - \nabla \ud(\py) \right) \nabla \fn(\udpx) + \nabla \ud(\py) \left( \nabla \fn(\udpx) - \nabla \fn(\udpy) \right).
\end{align*}
We first show the smoothness of the function. Taking the norm on both sides, for the specified $\alpha$, we have:
\begin{align*}
\| \nabla \udfn(\px) - \nabla \udfn(\py) \| & = \| \left( \nabla \ud(\px) - \nabla \ud(\py) \right) \nabla \fn(\udpx) + \nabla \ud(\py) \left( \nabla \fn(\udpx) - \nabla \fn(\udpy) \right) \| \\
& \leq \| \left( \nabla \ud(\px) - \nabla \ud(\py) \right) \nabla \fn(\udpx) \| + \| \nabla \ud(\py) \left( \nabla \fn(\udpx) - \nabla \fn(\udpy) \right) \|
\end{align*}
due to triangle inequality. We now bound both the terms on the right hand side. We have
\begin{align*}
    \| \left( \nabla \ud(\px) - \nabla \ud(\py) \right) \nabla \fn(\udpx) \| & \overset{(a)}{\leq} \| \nabla \ud(\px) - \nabla \ud(\py) \| \| \nabla \fn(\udpx) \| \\
    & = \| \left( \eye - \alpha \nabla^2 \fnht(\px) \right) - \left( \eye - \alpha \nabla^2 \fnht(\py) \right) \| \| \nabla \fn(\udpx) \| \\
    & = \alpha \| \nabla^2 \fnht(\px) - \nabla^2 \fnht(\py) \| \| \nabla \fn(\udpx) \| \\
    & \overset{(b)}{\leq} \alpha \rho \| \px - \py \| \| \nabla \fn(\udpx) \| \\
    & \overset{(c)}{\leq} \alpha \rho G \| \px - \py \|
\end{align*}
where (a) is due to the sub-multiplicative property of norms (Cauchy-Schwarz inequality), (b) is due to the Hessian Lipschitz property, and (c) is due to the Lipschitz property in function value. Similarly, we can bound the second term as
\begin{align*}
    \| \nabla \ud(\py) \left( \nabla \fn(\udpx) - \nabla \fn(\udpy) \right) \| & = \left\| \left( \eye - \alpha \nabla^2 \fnht(\py) \right) \Big( \nabla \fn(\udpx) - \nabla \fn(\udpy) \Big) \right\|
    \\
    & \overset{(a)}{\leq} (1-\alpha \mu) \| \nabla \fn(\udpx) - \nabla \fn(\udpy) \| \\
    & \overset{(b)}{\leq} (1-\alpha \mu) \beta \| \udpx - \udpy \| \\
    & \overset{(c)}{=} (1-\alpha \mu) \beta \| \ud(\px) - \ud(\py) \| \\
    & \overset{(d)}{\leq} (1-\alpha \mu) \beta (1-\alpha \mu) \| \px - \py \| \\
    & = (1-\alpha \mu)^2 \beta \|\px-\py\|
\end{align*}
Here, (a) is due to $\eye - \alpha \nabla^2 \fnht(\py)$ being symmetric, PSD, and $\lambda_{\max} \Big( \eye - \alpha \nabla^2 \fnht(\py) \Big) \leq 1-\alpha \mu$. Step (b) is due to smoothness of $\fn(\cdot)$. Step (c) is simply using our shorthand of $\udpx \equiv \ud(\px), \udpy \equiv \ud(\px)$. Finally, step (d) is due to Lemma~\ref{lemma:gd_contraction} on $\ud(\cdot)$.

Putting the previous pieces together, when $\alpha \leq \min \Big\{ \frac{1}{2\beta}, \frac{\mu}{8 \rho G} \Big\}$, we have that
\begin{align*}
\| \nabla \udfn(\px) - \nabla \udfn(\py) \| 
& \leq \| \left( \nabla \ud(\px) - \nabla \ud(\py) \right) \nabla \fn(\udpx) \| + \| \nabla \ud(\py) \left( \nabla \fn(\udpx) - \nabla \fn(\udpy) \right) \| \\
& \leq \alpha \rho G \| \px - \py \| + (1-\alpha \mu)^2 \beta \| \px - \py \| \\
& \leq \bigg( \frac{\mu}{8} + \beta  \bigg) \| \px - \py \| \\
& \leq \frac{9 \beta}{8} \| \px-\py \|
\end{align*}
and thus $\udfn(\cdot)$ is $\tilde{\beta}=\frac{9\beta}{8}$ smooth.

Similarly, for the lower bound, we have
\begin{align*}
\| \nabla \udfn(\px) - \nabla \udfn(\py) \| & = \| \left( \nabla \ud(\px) - \nabla \ud(\py) \right) \nabla \fn(\udpx) + \nabla \ud(\py) \left( \nabla \fn(\udpx) - \nabla \fn(\udpy) \right) \| \\
& \geq  \| \nabla \ud(\py) \left( \nabla \fn(\udpx) - \nabla \fn(\udpy) \right) \| - \| \left( \nabla \ud(\px) - \nabla \ud(\py) \right) \nabla \fn(\udpx) \|
\end{align*}
using the triangle inequality. We now again proceed to bound each term. We already derived an upper bound for the second term on the right side, and we require a lower bound for the first term.
\begin{align*}
     \| \nabla \ud(\py) \left( \nabla \fn(\udpx) - \nabla \fn(\udpy) \right) \| & = \left\| \left( \eye - \alpha \nabla^2 \fnht(\py) \right) \Big( \nabla \fn(\udpx) - \nabla \fn(\udpy) \Big) \right\| \\
     & \overset{(a)}{\geq} (1-\alpha \beta) \| \nabla \fn(\udpx) - \nabla \fn(\udpy) \| \\
     & \overset{(b)}{\geq} (1-\alpha \beta) \mu \| \udpx - \udpy \| \\
     & = (1-\alpha \beta) \mu \| \px - \alpha \nabla \fnht(\px) - \py + \alpha \nabla \fnht(\py) \| \\
     & \geq \mu (1-\alpha \beta) \Big( \| \px-\py \| - \alpha \| \nabla \fnht(\px) - \nabla \fnht(\py) \| \Big) \\
     & \overset{(c)}{\geq} \mu (1-\alpha \beta) \Big( \| \px-\py \| - \alpha \beta \| \px-\py \| \Big) \\
     & \geq \mu (1-\alpha \beta)^2 \| \px - \py \| \\
     & \geq \frac{\mu}{4} \| \px-\py \|
\end{align*}
Here (a) is due to $\lambda_{\min} \Big( \eye - \alpha \nabla^2 \fnht(\py) \Big) \geq 1-\alpha \beta$, (b) is due to strong convexity, and (c) is due to smoothness of $\fnht(\cdot)$. Using both the terms, we have that
\begin{align*}
\| \nabla \udfn(\px) - \nabla \udfn(\py) \| 
& \geq  \| \nabla \ud(\py) \left( \nabla \fn(\udpx) - \nabla \fn(\udpy) \right) \| - \| \left( \nabla \ud(\px) - \nabla \ud(\py) \right) \nabla \fn(\udpx) \| \\
& \geq \left( \frac{\mu}{4} - \frac{\mu}{8} \right) \|\px - \py \| \geq \frac{\mu}{8} \|\px-\py \|
\end{align*}
Thus the function $\udfn(\cdot)$ is $\tilde{\mu}=\frac{\mu}{8}$ strongly convex.
\end{proof}

\clearpage

\section{Additional Experimental Details}
\label{app:experiments}

For all experiments, we trained our FTML method with $5$ inner batch gradient descent steps with step size $\alpha=0.1$. We use an inner batch size of $10$ examples for MNIST and pose prediction and $25$ datapoints for CIFAR. 
Except for the CIFAR NML experiments, we train the convolutional networks using the Adam optimizer with default hyperparameters~\cite{adam}. We found Adam to be unstable on the CIFAR setting for NML and instead used SGD with momentum, with a momentum parameter of 0.9 and a learning rate of $0.01$ for the first $5$ thousand iterations, followed by a learning rate of $0.001$ for the rest of learning.

For the MNIST and CIFAR experiments, we use the cross entropy loss, using label smoothing with $\epsilon=0.1$ as proposed by~\citet{szegedy2016rethinking}. We also use this loss for the inner loss in the FTML objective.

In the MNIST and CIFAR experiments, we use a convolutional neural network model with $5$  convolution layers with $32$ $3 \times 3$ filters interleaved with batch normalization and ReLU nonlinearities. The output of the convolution layers is flattened and followed by a linear layer and a softmax, feeding into the output. In the pose prediction experiment, all models use a convolutional neural network with 4 convolution layers each with 16 $5 \times 5$ filters. After the convolution layers, we use a spatial soft-argmax to extract learned feature points, an architecture that has previously been shown to be effective for spatial tasks~\cite{e2e,singh2017gplac}. We then pass the feature points through 2 fully connected layers with 200 hidden units and a linear layer to the output. All layers use batch normalization and ReLU nonlinearities.

\end{document}